\documentclass[english]{article}
\usepackage{custom_tex}
\usepackage{authblk}
\usepackage{xcolor}
\allowdisplaybreaks

\newcommand{\bsu}{\boldsymbol{u}}
\newcommand{\bsx}{\boldsymbol{x}}
\newcommand{\bsy}{\boldsymbol{y}}
\newcommand{\bsz}{\boldsymbol{z}}

\newcommand{\real}{\mathbb{R}}

\newcommand{\rd}{\mathrm{\, d}}

\newcommand{\diag}{{\mathrm{diag}}}
\newcommand{\tr}{{\mathrm{tr}}}

\newcommand{\var}{{\mathrm{Var}}}

\newcommand{\wt}{\widetilde}

\DeclareMathOperator*{\argmin}{argmin}

\renewcommand{\ge}{\geqslant}
\renewcommand{\le}{\leqslant}

\newcommand{\dnorm}{\mathcal{N}}

\newcommand{\dustd}{\mathbf{U}} 

\newcommand{\e}{{\mathbb{E}}} 

\title{Quasi-Monte Carlo Quasi-Newton for variational Bayes}
\author{Sifan Liu}
\author{Art B. Owen}
\affil{Department of Statistics, Stanford University}
\date{April 2021}

\renewcommand{\le}{\leqslant}
\renewcommand{\leq}{\leqslant}
\renewcommand{\ge}{\geqslant}
\renewcommand{\geq}{\geqslant}

\newcommand{\simind}{\stackrel{\mathrm{ind}}{\sim}}

\usepackage[colorlinks=true,linkcolor=blue,citecolor=blue,urlcolor=blue,breaklinks]{hyperref}
\usepackage[authoryear,round]{natbib}
\newtheorem{theorem}{Theorem}[section]

\newtheorem{lemma}{Lemma}

\newtheorem{remark}{Remark}[section]
\newcommand{\BlackBox}{\rule{1.5ex}{1.5ex}}  

\newcommand{\param}{\theta}
\newcommand{\Param}{\Theta}
\newcommand{\bvhk}{\mathrm{BVHK}}

\newcommand{\giv}{\!\mid\!} 

\newcommand{\bigz}{{\mathcal{Z}}} 
\newcommand{\sumdot}{\text{\tiny$\bullet$}}
\newcommand{\pr}{\mathbb{P}}

\ifdefined\proof
    \renewenvironment{proof}{\par\noindent{\bf Proof\ }}{\hfill\BlackBox\\[2mm]}
\else
    \newenvironment{proof}{\par\noindent{\bf Proof\ }}{\hfill\BlackBox\\[2mm]}
\fi
\begin{document}
\maketitle

\begin{abstract}
Many machine learning problems optimize an objective that
must be measured with noise.  The primary method is a first
order stochastic gradient descent using one or more
Monte Carlo (MC) samples at each step.
There are settings where ill-conditioning
makes second order methods such as L-BFGS more effective.
We study the use of randomized quasi-Monte Carlo (RQMC) sampling
for such problems.  When MC sampling has a
root mean squared error (RMSE) of $O(n^{-1/2})$ then
RQMC has an RMSE of $o(n^{-1/2})$ that can be close to $O(n^{-3/2})$
in favorable settings. We prove that improved sampling accuracy
translates directly to improved optimization.  In our empirical
investigations for variational Bayes,  using RQMC with stochastic L-BFGS greatly
speeds up the optimization, and sometimes finds a better
parameter value than MC does.
\end{abstract}

\section{Introduction}
Many practical problems take the form
\begin{align}\label{eq:simopt}
\min_{\param\in\Param\subseteq\real^d}F(\param)
\quad\text{where}\quad F(\param) = \e( f(\bsz;\param))
\end{align}
and $\bsz$ is random vector with a known distribution $p$.
Classic problems of simulation-optimization \citep{andr:1998}
take this form and recently it has become very important
in machine learning with variational Bayes (VB) \citep{blei:kucu:macu:2017}
and Bayesian optimization \citep{frazier2018tutorial} being prominent examples.

First order optimizers \citep{beck:2017} use a sequence of steps
$\param_{k+1}\gets \param_k-\alpha_k\nabla F(\param_k)$
for step sizes $\alpha_k>0$ and an operator $\nabla$
that we always take to be the gradient with respect to $\param$.
Very commonly, neither $F$ nor its gradient is available at reasonable
computational cost and a Monte Carlo (MC) approximation is used instead.
The update
\begin{align}\label{eq:basicsgd}
\param_{k+1}\gets \param_k-\alpha_k\nabla \hat F(\param_k)
\quad\text{for}\quad
  \nabla \hat F(\param_k) = \frac1n\sum_{i=1}^ng(\bsz_i;\param)
\quad\text{and}\quad
g(\bsz;\param)=\nabla f(\bsz;\param)
\end{align}
for $\bsz_i\simiid p$
is a simple form of stochastic gradient descent (SGD) \citep{duchi2018introductory}.
There are many versions of SGD, notably AdaGrad \citep{duchi2011adaptive}
and Adam \citep{kingma2014adam} which are prominent in optimization
of neural networks among other learning algorithms.
Very often SGD involves sampling a mini-batch of observations.
In this paper we consider SGD that samples instead some
random quantities from a continuous distribution.

While a simple SGD is often very useful, there are settings
where it can be improved.
SGD is known to have slow convergence when the Hessian
of $F$ is ill-conditioned \citep{bottou2018optimization}.
When $\param$ is not very high dimensional, the second
order methods such as Newton iteration using
exact or approximate Hessians can perform better.
Quasi-Newton methods such as BFGS and L-BFGS \citep{nocedal2006numerical}
that we describe in more details below
can handle much higher
dimensional parameters than Newton methods can while
still improving upon first order methods.
We note that
quasi second-order methods cannot be proved to have better convergence
rate than SGD. See \cite{agarwal2012information}.
They can however have a better implied constant.


A second difficulty with SGD is that MC sampling to estimate
the gradient can be error prone or inefficient.
For a survey of MC methods to estimate a gradient
see \cite{mohamed2020monte}.
Improvements based on variance reduction methods
have been adopted to improve SGD.
For instance \cite{paisley2012variational}
and \cite{miller2017reducing} both employ control
variates in VB.
Recently, randomized quasi-Monte Carlo (RQMC)
methods, that we describe below,
have been used in place of MC to improve upon
SGD. Notable examples are \cite{balandat2020botorch}
for Bayesian optimization and \cite{buchholz2018quasi} for VB.
RQMC can greatly improve the accuracy with which
integrals are estimated. The theory in \cite{buchholz2018quasi}
shows how the improved integration accuracy from
RQMC translates into faster optimization for SGD.
The primary benefit is that RQMC can use a much
smaller value of $n$.
This is also seen empirically in \cite{balandat2020botorch} for Bayesian
optimization. 

Our contribution is to combine RQMC with a second order
limited memory method known as L-BFGS. We show theoretically
that improved integration leads to improved optimization.
We show empirically for some VB examples
that the optimization is improved.   We find that RQMC
regularly allows one to use fewer samples per iteration and in
a crossed random effects example it found a better solution
than we got with MC.

At step $k$ of our stochastic optimization
we will use some number $n$ of sample values,
$\bsz_{k,1},\dots,\bsz_{k,n}$.  It is notationally
convenient to group these all together into
$\bigz_k = (\bsz_{k,1},\dots,\bsz_{k,n})$.
We also write
\begin{align}\label{eq:defbarg}
\bar g(\bigz_k;\param) = \frac1n\sum_{i=1}^ng(\bsz_{k,i};\param)
\end{align}
for gradient estimation at step $k$.

The closest works to ours are \cite{balandat2020botorch}
and \cite{buchholz2018quasi}.  Like \cite{balandat2020botorch}
we incorporate RQMC into L-BFGS. Our
algorithm differs in that we take fresh RQMC samples at
each iteration where they had a fixed sample of size
$n$ that they used in a `common random numbers'
approach. They prove consistency as $n\to\infty$
for both MC and RQMC sampling.  Their proof
for RQMC required the recent strong law of large
numbers for RQMC from \cite{sllnrqmc}.
Our analysis incorporates the geometric decay
of estimation error as the number $K$ of
iterations  increases, similar to that used
by \cite{buchholz2018quasi} for SGD.
Our error bounds include sampling variances
through which RQMC brings an advantage.

An outline of this paper is as follows.
Section~\ref{sec:optim} reviews the optimization methods we need.
Section~\ref{sec:rqmc} gives basic properties
of scrambled net sampling, a form of RQMC.
Section~\ref{sec:theory} presents our main theoretical
findings.
The optimality gap after $K$ steps of
stochastic quasi-Newton is $F(\param_K)-F(\param^*)$
where $\param^*$ is the optimal value.
Theorem~\ref{thm: f} bounds
the expected optimality gap by a term that decays exponentially in $K$
plus a second term that is linear in
a measure of sampling variance.
RQMC greatly
reduces that second non-exponentially decaying term
which will often dominate.
A similar bound holds for tail probabilities of the optimality gap.
Theorem~\ref{thm: param}
obtains a comparable bound for $\e(\Vert\param_K-\param^*\Vert^2)$.
Section~\ref{sec:vb} gives numerical examples
on some VB problems.
In  a linear regression problem where the optimal
parameters are known, we verify that RQMC converges
to them at an improved rate in $n$. In logistic regression,
crossed random effects and variational autoencoder
examples we see the second order methods greatly
outperform SGD in terms of wall clock time
to improve $F$.
Since the true parameters
are unknown we cannot compare accuracy
of MC and RQMC sampling algorithms for those examples.
In some examples RQMC finds a better ELBO than
MC does when both use SGD, but the BFGS
algorithms find yet better ELBOs.
Section~\ref{sec:discuss} gives some conclusions.
The proofs of our main results are in an appendix.

\section{Quasi-Newton optimization}\label{sec:optim}

We write $\nabla^2 F(\param)$ for the Hessian matrix of $F(\param)$.
The classic Newton update is
\begin{align}\label{eq:cnewton}
\param_{k+1}\gets \param_k-(\nabla^2F(\param_k))^{-1}\nabla F(\param_k).
\end{align}
Under ideal circumstances it converges quadratically
to the optimal parameter value $\param^*$. That is
$\Vert\param_{k+1}-\param^*\Vert=O(\Vert\param_{k}-\param^*\Vert^2)$.
Newton's method is unsuitable for the problems we consider here because
forming $\nabla^2F\in\real^{d\times d}$ may take too much space and solving
the equation in \eqref{eq:cnewton} can cost $O(d^3)$ which is prohibitively expensive.
The quasi-Newton methods we consider do not form explicit
Hessians.  We also need to consider stochastic versions of them.

\subsection{BFGS and L-BFGS}

The BFGS method is named after four independent discoverers: Broyden,
Fletcher, Goldfarb and Shanno.
See \citet[Chapter 6]{nocedal2006numerical}.
BFGS avoids explicitly computing and inverting the Hessian of $F$.
Instead it maintains at step $k$ an approximation $H_k$ to the
inverse of $\nabla^2F(\param_k)$.
After an initialization such as setting $H_1$ to the identity
matrix, the algorithm updates $\param$ and $H$ via
\begin{align*}
\param_{k+1} & \gets \param_k -\alpha_k H_k\nabla F(\param_k),\quad\text{and}\\
H_{k+1} & \gets \Bigl( I-\frac{s_ky_k\tran}{s_k\tran y_k}\Bigr)
H_k \Bigl( I-\frac{s_ky_k\tran}{s_k\tran y_k}\Bigr) +
\frac{s_ks_k\tran}{s_k\tran y_k}
\end{align*}
respectively, where
\begin{align*}
s_k =\param_{k+1}-\param_k
\quad\text{and}\quad y_k = \nabla F(\param_{k+1})-\nabla F(\param_k).
\end{align*}
The stepsize $\alpha_k$ is found by a line search.

Storing $H_k$ is a burden and the limited-memory BFGS (L-BFGS) algorithm
of \cite{nocedal1980updating} avoids forming $H_k$ explicitly.
Instead it computes
$H_k\nabla F(\param_k)$ using a recursion
based on the $m$ most recent $(s_k,y_k)$ pairs.
See \citet[Algorithm 7.4]{nocedal2006numerical}.

\subsection{Stochastic quasi-Newton}

Ordinarily in quasi-Newton algorithms
the objective function remains constant through
all the iterations. In stochastic quasi-Newton algorithms
the sample points change at each iteration which is like
having the objective function $F$ change in a random
way at iteration $k$.
Despite this \cite{bottou2018optimization} find that
quasi-Newton can work well in simulation-optimization
with these random changes.

We will need to use sample methods to evaluate
gradients. Given $\bigz=(\bsz_1,\dots,\bsz_n)$ the
sample gradient is $\bar g(\bigz;\param)$ as given at \eqref{eq:defbarg}.
Stochastic quasi-Newton algorithms like the one we study
also require randomized Hessian information.

\cite{byrd2016stochastic} develop a stochastic L-BFGS algorithm
in the mini-batch setting.
Instead of adding every correction pair $(s_k,y_k)$ to the buffer after each iteration,
their algorithm updates the buffer every $B$ steps using the averaged correction pairs.
Specifically, after every $B$ iterations, for $k=t\times B$,
it computes the average
parameter value $\bar\param_t=B^{-1}\sum_{j=k-B+1}^k\param_j$
over the most recent $B$ steps.
It then computes the correction pairs by $s_t=\bar\param_t-\bar\param_{t-1}$
and
$$y_t=
\bar g(\wt\bigz_t;\bar\param_t)-\bar g(\wt\bigz_t;\bar\param_{t-1})
=\frac1n\sum_{i=1}^n
\bigl(g(\bar\param_t,\tilde\bsz_{t,i})-g(\bar\param_{t-1},\tilde\bsz_{t,i})\bigr),$$
and adds the  pair $(s_t,y_t)$ to the buffer.
Here $\wt\bigz_t=(\tilde\bsz_{t,1},\dots,\tilde \bsz_{t,n})$
is a random sample of size $n=n_h$
Hessian update samples.
The samples $\wt\bigz_t$ for $t\ge1$ are
completely different from and independent of $\bigz_k$ for $k\ge1$
used to update gradient estimates at step $k$.
This method is called SQN (stochastic quasi-Newton). As suggested by the
authors, $B$ is often taken to be 10 or 20. So one can afford to use relatively
large $n_h$ because the amortized average number of
gradient evaluations per iteration is $n_g + 2n_h/B$.

The objective function from \cite{byrd2016stochastic}
has the form
$F(\param)=(1/N)\sum_{i=1}^Nf(\bsx_i;\param)$, and $\{\bsx_1,\ldots,\bsx_N\}$
is a fixed training set. At each iteration, their random sample is a
subset of $n\ll N$ points in the dataset, not a sample generated from
some continuous distribution.
However, it is straightforward  to adapt their algorithm to our setting. The details
are in Algorithm \ref{alg: qmc-l-bfgs} in Section~\ref{sec:rqmc}.

\subsection{Literature review}
Many authors have studied how to use L-BFGS in stochastic settings.
{\cite{bollapragada2018progressive} proposes several techniques for stochastic L-BFGS,
including increasing the sample size with iterations (progressive batching), choosing
the initial step length for backtracking line search so that the expected value of the
objective function decreases, and computing the correction pairs using overlapping
samples in consecutive steps \citep{berahas2016multi}.
}

{Another way to prevent noisy updates is by a lengthening strategy from
\cite{xie2020analysis} and \cite{shi2020noise}. Classical BFGS would use the
correction pair $(\alpha_kp_k,\,g(\theta_k+\alpha_kp_k,\bsz) - g(\theta_k,\bsz) )$,
where $p_k=-H_kg(\theta_k,\bsz)$ is the update direction, and $\bsz$ encodes the randomness in estimating the gradients. Because the correction pairs may be dominated by noise, \cite{xie2020analysis} suggested the lengthening correction pairs
\[
(s_k,\, y_k)=(\beta_kp_k, \, g(\theta_k+\beta_kp_k) - g(\theta_k)),\quad \text{where }\beta_k\geq\alpha_k.
\]
\cite{shi2020noise} propose to choose $\alpha_k$ by the Armijo-Wolfe condition, while choosing $\beta_k$ large enough so that $[g(\theta_k+\beta_kp_k)-g(\theta_k)]\tran p_k/\|p_k\|$ is sufficiently large.
}

{\cite{gower2016stochastic} utilizes sketching strategies to update the inverse Hessian
approximations by compressed Hessians. \cite{moritz2016linearly} combines the stochastic
quasi-Newton algorithm in \cite{byrd2016stochastic} and stochastic
variance reduced gradients (SVRG)
\citep{johnson2013accelerating} by occasionally computing the gradient using
a full batch. }

{\cite{balandat2020botorch} also applied RQMC with L-BFGS in Bayesian optimization.
At each step of Bayesian optimization, one needs to maximize the acquisition function
of the form $\alpha(\theta):=\EE{\ell(g(\theta))}$, where $g$ is a Gaussian process and
$\ell$ is a loss function. They use the sample average approximation
$\hat\alpha(\theta):=(1/n)\sum_{i=1}^n\ell(\xi_i(\theta))$, where
$\xi_i(\theta)\sim g(\theta)$. They prove that the maximizer of $\hat\alpha$
converges to that of $\alpha$ when $n\goinf$ for both MC and RQMC
sampling under certain conditions. The RQMC result relies on the recent discovery
of the strong law of large numbers for RQMC \citep{sllnrqmc}.
}


\section{Scrambled net sampling}\label{sec:rqmc}

Scrambled nets are a form of RQMC sampling.
We begin by briefly describing plain
quasi-Monte Carlo (QMC) sampling.
QMC is most easily
described for computing expectations of $f(\bsu)$
for $\bsu\sim\dustd[0,1]^s$.
In our present context $f(\cdot)$ will be a component of
$g(\cdot;\param)$ and not the same as the $f$ in equation~\eqref{eq:simopt}.
In practice we must ordinarily transform the random variables
$\bsu$ to some other distribution such as a Gaussian
via $\bsx=\psi(\bsu)$.
We suppose that $\psi(\bsu)\sim p$ for
some transformation $\psi(\cdot)$.
The text by \cite{devr:1986} has many examples of such transformations.
Here we subsume any such transformation $\psi(\cdot)$ into the
definition of $f$.

In QMC sampling, we estimate $\mu =\int_{[0,1]^s}f(\bsu)\rd\bsu$ by
$\hat\mu = (1/n)\sum_{i=1}^nf(\bsu_i)$, just like in MC sampling
except that distinct points $\bsu_i$ are chosen so that the
discrete uniform distribution  $\dustd\{\bsu_1,\dots,\bsu_n\}$ is made
very close to the continuous $\dustd[0,1]^s$ distribution.
The difference between these two distributions can be quantified
in many ways, called discrepancies \citep{chen:sriv:trav:2014}.
For a comprehensive treatment of QMC see \cite{dick:pill:2010} or \cite{nied:1992}
or \cite{dick:kuo:sloa:2013}.

When $f$ is of bounded variation in the sense of Hardy
and Krause (BVHK) (see \cite{variation}) then QMC attains the asymptotic
error rate $|\hat\mu-\mu| = O( n^{-1}\log(n)^{s-1})=O(n^{-1+\epsilon})$
for any $\epsilon>0$.
QMC is deterministic and to get practical error estimates RQMC
methods were introduced. In RQMC, each individual $\bsu_i\sim\dustd[0,1]^d$
while collectively $\bsu_1,\dots,\bsu_n$ still retain the low discrepancy
property.  Uniformity of $\bsu_i$ makes RQMC unbiased: $\e(\hat\mu)=\mu$.
Then if $f\in\bvhk$ we get an RMSE of $O(n^{-1+\epsilon})$.
The whole RQMC process can then be replicated independently
to quantify uncertainty.
See \cite{cran:patt:1976} and \cite{rtms} for methods and a survey in \cite{lecu:lemi:2002}.

Scrambled net sampling \citep{rtms} is a form of RQMC that operates
by randomly permuting the bits (more generally digits) of QMC
methods called digital nets.  The best known are those of \cite{sobo:1969}
and \cite{faur:1982}.
In addition to error estimation, scrambled nets give the user
some control over  the powers of $\log(n)$ in the QMC rate
and also extend the domain of QMC from Riemann integrable functions
of which BVHK is a subset to much more general
functions including some with integrable singularities.
For any integrand $f\in L^2[0,1]^s$,  MC has RMSE $O(n^{-1/2})$
while scrambled nets have RMSE $o(n^{-1/2})$
without requiring $f\in\bvhk$ or even that $f$
is Riemann integrable \citep{snetvar}.
For fixed $n$, each construction of scrambled nets
has a `gain constant' $\Gamma<\infty$
so that the RMSE is below
$\sqrt{\Gamma} n^{-1/2}$ for any $f\in L^2[0,1]^s$.
This effectively counters  the powers of $\log(n)$.
For smooth enough $f$, an error cancellation
phenomenon for scrambled nets yields
an RMSE of $O(n^{-3/2}\log(n)^{(s-1)/2})
=O(n^{-3/2+\epsilon})$ \citep{smoovar,yue:mao:1999,localanti}.
The logarithmic powers here cannot `set in' until
they are small enough to obey the $\Gamma^{1/2}n^{-1/2}$ upper bound.
Some forms of scrambled net sampling satisfy
a central limit theorem \citep{loh:2003,basu:mukh:2017}.

Very recently, a strong law of large numbers
$$
\Pr\Bigl( \lim_{n\to\infty}\hat\mu_n =\mu\Bigr)=1
$$
has been proved for scrambled net sampling
assuming only that $f\in L^{1+\delta}[0,1]^s$ for
some $\delta>0$ \citep{sllnrqmc}.  The motivation for this
result was that \cite{balandat2020botorch}
needed a strong law of large numbers  to prove consistency for their use for
scrambled nets in Bayesian optimization.

Figure~\ref{fig:sobolthing} graphically compares MC, QMC and RQMC points
for $s=2$.  The underlying QMC method is a Sobol' sequence
using `direction numbers' from \cite{joe:kuo:2008}.
We can see that MC points leave voids and create clumps.
The QMC points are spread out more equally and show
a strong diagonal structure.  The RQMC points satisfy
the same discrepancy bounds as the QMC points do but
have broken up some of the structure.

\begin{figure}
\centering
\includegraphics[width=.9\hsize]{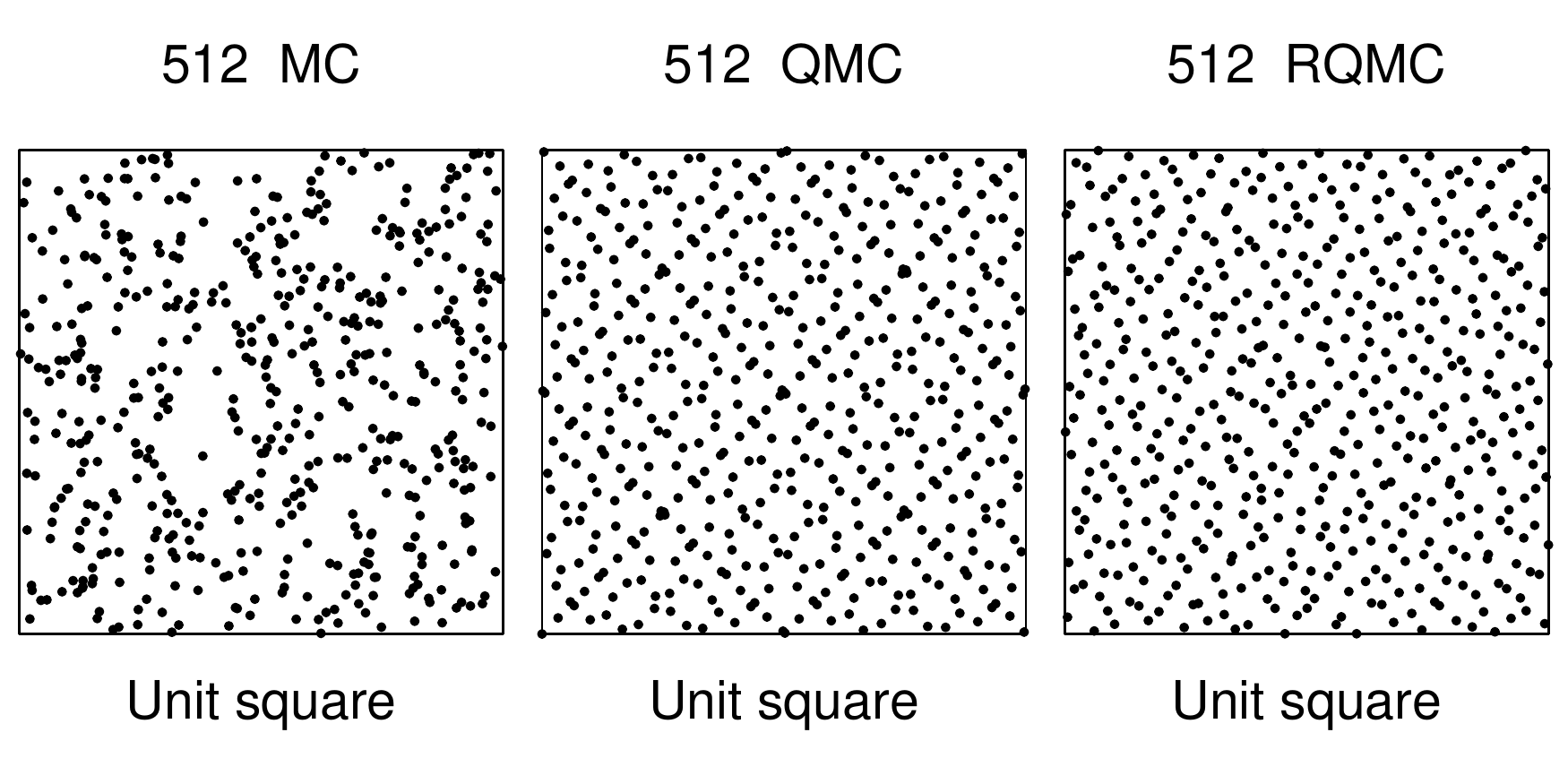}
\caption{\label{fig:sobolthing}
Each panel shows $512$ points in $[0,1]^2$.
From left to right they are plain MC,
Sobol' points and scrambled Sobol' points.
From \cite{sllnrqmc}: Copyright \textcopyright\ 2021 Society for Industrial and Applied Mathematics.  Reprinted with permission.  All rights reserved.
}
\end{figure}

In favorable settings the empirical behaviour of QMC and RQMC for realistic $n$
can be as good as their asymptotic rates.
In less favorable settings RQMC can
be like MC with some reduced variance.
The favorable integrands are those where $f$ is nearly additive or
at least dominated by sums of only a few of their
inputs at a time.  See \cite{cafl:moro:owen:1997} for
a definition of functions of `low effective dimension'
and \cite{dick:kuo:sloa:2013} for a survey of work on reproducing
kernel Hilbert spaces of favorable integrands.

We propose to combine the stochastic quasi-Newton method with RQMC
samples to create a randomized quasi-stochastic quasi-Newton  (RQSQN)
algorithm. 
At the $k$'th iteration, we draw an independently scrambled
refreshing sample
$\bigz_k=(\bsz_{k,1},\dots,\bsz_{k,n_g})$
of size $n=n_g$ via
RQMC to compute the gradient estimator $\bar g(\bigz_k;\param_k)$.
Then we find the descent direction $H_k\bar g(\bigz_k;\param_k)$
using an L-BFGS two-loop recursion.
Then we update the solution by
$$\param_{k+1}\gets\param_k-\alpha_kH_k\bar g(\bigz_k;\param_k).$$
Here $\alpha_k$
 may be found by line search with the Wolfe condition when using L-BFGS. See Chapter 3.1 in \cite{nocedal2006numerical}.

Algorithm \ref{alg: qmc-l-bfgs} shows pseudo-code for
an RQMC version of SQN based on L-BFGS.
It resembles the SQN algorithm
in \cite{byrd2016stochastic}, except that the random samples are
drawn by using RQMC instead of being sampled without
replacement from a finite data set.
Note that we don't compute the Hessian directly.
We either compute the Hessian-vector product
$\nabla^2 f(\bigz_t;\bar\param_t)s_t$ or use
gradient differences $\nabla f(\bar\theta_t)-\nabla f(\bar\theta_{t-1})$.



\begin{algorithm}[t]
\SetKwInOut{Input}{Input}
\SetKwInOut{Output}{Output}
\Input{Initialization $\param_1$, buffer size $m$, Hessian update interval $B$, sample sizes $n_g$ and $n_h$ for estimating gradient and updating buffer}
 \Output{Solution $\param$}
$t\gets-1$\;
 \For{$k=1,2,\ldots$}{
 Take an RQMC sample $\bsz_{k,1},\ldots,\bsz_{k,n_g}\sim p$\;
 Calculate the gradient estimator
$g_k\gets\bar g(\bigz_k;\param_k)=
\frac{1}{n_g}\sum_{i=1}^{n_g}g(\bsz_{k,i};\param_k)$\;
 \eIf{$t<1$}{
 $\param_{k+1}\gets\param_k-\alpha_k g_k$\;
 }
 {
 Find $H_tg_k$ by the two-loop recursion with memory size $m$\;
 Find $\alpha_k$ by line search\;
 Update $\param_{k+1}\gets\param_k-\alpha_kH_tg_k $\;
 }

\If{$\mod(k,B)=0$}{
  $t\gets t+1$\;
  $\bar\param_t\gets B^{-1}\sum_{j=k-B+1}^k\param_j$\;
  \If{$t>0$}{

  Take an RQMC sample $\tilde\bsz_{t,1},\ldots,\tilde\bsz_{t,n_h}\sim p$\;
  Add $s_t=(\bar\param_t-\bar\param_{t-1})$, $y_t=\bar g(\wt\bigz_t;\bar\param_t)=\frac{1}{n_h}\sum_{i=1}^{n_h}\nabla^2 f(\tilde\bsz_{t,i};\bar\param_t)s_t$ to the buffer\;
  }
}{}
}
 \caption{RQMC-SQN}
\label{alg: qmc-l-bfgs}
\end{algorithm}

\section{Theoretical guarantees}\label{sec:theory}


In this section, we study the convergence rate of a general quasi-Newton iteration
based on $n$ sample points $\bsz_{k1},\dots,\bsz_{kn}\sim p$
at stage $k$.
The algorithm iterates as follows
\begin{align}\label{eq:sqnupdate}
\param_{k+1}&\gets\param_k-\alpha_kH_k\nabla \bar f(\bigz_k;\param_k),\quad\text{where}
\\
\nabla\bar f(\bigz_k;\param_k) &= \frac1n\sum_{i=1}^n \nabla f(\bsz_{ki};\param_k)
= \bar g(\bigz_k;\param_k).
\notag
\end{align}
Here $H_k$ is an approximate inverse Hessian.
We let $F(\param)=\e(f(\bsz_{ki};\param))$.
We assume that the gradient estimator $g(\bigz_{k};\param_k)$
is unbiased conditionally on $\param_k$, i.e.,
$\e(g(\bigz_{k};\param_k) \mid\param_k) =\nabla F(\param_k)$.
The $\param_k$ are random because they depend on $\bigz_{k'}$
for $k'<k$.


The Hessian estimates $H_k$ for $k>1$ are also random because they depend
directly on some additional inputs $\tilde\bigz_{t}$
for those Hessian update epochs $t$ which occur prior to step $k$.
They also depend on $\bigz_{k'}$ for $k'<k$ because
those random variables affect $\param_{k'}$.
In our algorithms $H_k$ is independent of $\param_k$.
For RQMC this involves fresh rerandomizations of the underlying
QMC points at step $k$.  The alternative of `going deeper'
into a given RQMC sequence using points $(k-1)n+1$
through $kn$ would not satisfy independence.
While it might possibly perform better it is harder to analyze
and we saw little difference empirically in doing that.
Our theory requires regularity of the problem as follows.

\begin{assumption}\label{ass:regularity}
We impose these three conditions:
\begin{enumerate}[(a)]
\item Strong convexity. For some $c>0$,
$$F(\param')\geq F(\param)+\nabla F(\param)\tran(\param'-\param)+\frac{c}{2}\|\param'-\param\|_2^2\quad\text{for all}\quad\param,\param'\in\Param.$$

\item Lipschitz continuous objective gradients. For some $L<\infty$,
$$\|\nabla F(\param)-\nabla F(\param')\|\leq L\|\param-\param'\|
\quad\text{for all}\quad\param,\param'\in\Param.$$

\item Bounded variance of the gradient estimator. For some $M<\infty$,
$$\tr\bigl(\var(\bar g(\bigz_k;\param _k)\giv\param_k)\bigr)\leq M\quad\text{for all}\quad k\geq1.$$
\end{enumerate}
\end{assumption}

These are standard assumptions in the study of SGD. For example, see Assumptions 4.1 and 4.5 of \cite{bottou2018optimization}.
Strong convexity implies that there exists a unique minimizer $\param^*$ to estimate.
We write  $F^*=F(\param^*)$, for the best possible value of $F$.
We must have some smoothness and Lipshitz continuity
is a mild assumption.  The quantity $M$ will prove to be
important below. We can get a better $M$ from RQMC than from MC.
The other way to reduce $M$ is to increase $n$.  When RQMC has an
advantage it is because it gets a smaller $M$ for the same $n$,
or to put it another way, it can get comparable $M$ with smaller $n$.

For two symmetric matrices $A$ and $B$, $A\preccurlyeq B$ means that $B-A$ is positive semi-definite.
We denote the spectral norm of $A$ by $\|A\|$.

\begin{theorem}[Convergence of optimality gap]\label{thm: f}


Suppose that our simulation-optimization problem
satisfies the regularity conditions in  Assumption \ref{ass:regularity}.
Assume that we run updates as in equation~\eqref{eq:sqnupdate}
where  the approximate inverse Hessian matrices $H_k$ satisfy
$h_1 I\preccurlyeq H_k\preccurlyeq h_2 I$
for some $0<h_1\leq h_2$ and all $k\geq1$.
Next, assume constant step sizes $\alpha_k=\alpha$
with $0<\alpha\le h_1/(Lh_2^2)$.
Then for every $K\geq1$
\begin{align}
\e(F(\param_{K})-F^*)&\leq(1-\alpha ch_1)^{K}(F(\param_0)-F^*)+\frac{\alpha Lh_2^2}{2ch_1}M.
\label{eq: bound expectation}
\end{align}
Furthermore, if $\|g(\param,\bsz)\|\leq C$ for some constant $C$ for all $\param$ and $\bsz$, then for any $\ep>0$
\begin{align}\label{eq: bound finite sample}
F(\param_{K})-F^*&\leq(1-\alpha c h_1)^{K}(F(\param_0)-F^*)
+\frac{\alpha L h_2^2}{2ch_1}M
+C^2\sqrt{\frac{2\alpha}{ch_1}}\bigl(h_2-L\alpha h_1^2+{h_1}{}\bigr)\ep
\end{align}
holds with probability at least $1-e^{-\ep^2}$.
\end{theorem}
\begin{proof}
See Section \ref{sec: proof 1} in the Appendix.
\end{proof}
\begin{remark}
As $K\goinf$ the
expected optimality gap is no larger than $[(\alpha Lh_2^2)/(2ch_1)]\times M$.
The variance bound $M=M(n)$ (for $n$ points $\bsz_{k,i}$)
depends on the sampling method we choose.
From the results in Section~\ref{sec:rqmc},
scrambled nets can reduce $M$ from $O(1/n)$ for MC to $o(1/n)$
attaining $O(n^{-2+\epsilon})$ or even
$O(n^{-3+\epsilon})$ in favorable cases. When $M<\infty$
for MC then $M\le \Gamma M$ for scrambled net RQMC,
which limits the harm if any that could come from RQMC.
\end{remark}

\begin{remark}
By Lemma 3.1 of \cite{byrd2016stochastic}, the L-BFGS iteration satisfies
$h_1 I\preccurlyeq H_k\preccurlyeq h_2 I$ under weaker conditions
than we have in Theorem \ref{thm: f}.
We can replace the  bound $\Vert g(\bsz;\param)\Vert\le C$
by $\e(\|g(\bsz;\param)\|^2)\leq C^2<\infty$.
\end{remark}

The following theorem states the convergence rate of $\|\param_k-\param^*\|$.
\begin{theorem}[Convergence of variables]
\label{thm: param}
Under the conditions of Theorem \ref{thm: f}, suppose that
$$
\frac{c}L>\frac{h_2-h_1}{h_2+h_1}\quad\text{and}\quad
0<\alpha_k < \frac{(h_1+h_2)c-(h_2-h_1)L}{2L^2h_2^2}.$$
Then for every $K\ge 1$,
\begin{align*}
\e(\Vert \param_{K}-\param^*\Vert^2)
\leq \bigl(1-\alpha^2h_2^2L^2\bigr)^K\Vert\param_0-\param^*\Vert^2+M/L^2.
\end{align*}
\end{theorem}
\begin{proof}
See Section \ref{sec: proof 2} in the Appendix.
\end{proof}

\begin{remark}
When $K\goinf$, the limiting expected squared error is bounded by
$\alpha^2h_2^2M$. Once again the potential gain from RQMC is in reducing $M$
compared to MC or getting a good $M$ with smaller~$n$.
\end{remark}

\section{Variational Bayes}\label{sec:vb}

In this section we investigate quasi-Newton quasi-Monte Carlo optimization
for some VB problems.
Variational Bayes begins with a posterior distribution $p(\bsz\giv\bsx)$
that is too inconvenient to work with. This usually means
that we cannot readily sample $\bsz$.
We turn instead to a distribution $q_{\param}$ for $\param\in\Param$
from which we can easily sample $\bsz\sim q_{\param}$.
We now want to make a good choice of $\param$ and
in VB the optimal value $\param^*$
is taken to be the minimizer of the KL divergence
between distributions:
\begin{align*}
\param^* = \argmin_{\param\in\Param}\,\KL( q_\param(\bsz)\, \|\, p(\bsz\giv\bsx)).
\end{align*}

In this section we are using the symbol $p(\cdot)$ as a generic
probability distribution.  It is the distribution of whatever random variable's
symbol is inside the parentheses and not necessarily the same $p$ from the introduction.
We use $\e_\param(\cdot)$ to denote expectation with respect to $\bsz\sim q_\param$.

We suppose that $\bsz$ has a prior distribution $p(\bsz)$. Then Bayes rule gives
\begin{align*}
\KL(q_\param(\bsz)\, \|\, p(\bsz\giv\bsx))&=
\e_{\param}\bigl({\log q_\param(\bsz)-\log p(\bsz\giv\bsx)}\bigr)\\
&=\e_{\param}\Bigl({\log q_\param(\bsz)-\log \frac{p(\bsx\giv\bsz)p(\bsz)}{p(\bsx)}}\Bigr)\\
&=\KL(q_\param(\bsz)\,\|\, p(\bsz))-\e_{\param}(p(\bsx\giv\bsz))+\log p(\bsx).
\end{align*}
The last term does not depend on $\param$ and so
we may minimize $\KL(\cdot\,\|\,\cdot)$ by maximizing
  \[
    \calL(\param)=\e_\param({\log p(\bsx\giv\bsz)})-\KL(q_\param(\bsz)\,\|\, p(\bsz)).
  \]
This $\calL(\cdot)$ is known as the evidence lower bound (ELBO).
The first term $\e_\param({\log p(\bsx\giv\bsz)})$ expresses
a preference for $\param$ having a large value of the likelihood of
the observed data $\bsx$ given the latent data $\bsz$.
The second term $-\KL(q_\param(\bsz\,\|\, p(\bsz))$
can be regarded as a regularization, penalizing  parameter values
for which $q_\param(\bsz)$ is too far from the prior distribution $p(\bsz)$.

To optimize $\calL(\param)$ we need $\nabla\calL(\param)$.
It is usual to choose a family $q_\param$ for which
$\KL(\cdot\,\|\,\cdot)$ and its gradient are analytically tractable.
We still need to estimate the gradient of the first term, i.e.,
\begin{align*}
\nabla\, \e_{\param}({\log p(\bsx\giv\bsz)}).
\end{align*}
One method is to use the score function $\nabla\log p_\param(z)$ and Fisher's identity
\begin{align*}
\nabla \e_\param({f(\bsz)})=\e_\param({f(\bsz)\nabla\log p_\param(\bsz)})
\end{align*}
with $f(\cdot) = \log(p(\bsx\giv\cdot))$.
Unfortunately an MC strategy based on this approach can suffer from large variance.

The most commonly used method is to write
the parameter $\param$ as a function of some underlying
common random variables.
This is known as the reparameterization trick
\citep{kingma2014auto}. Suppose that there is a base distribution $p_0$ and
a transformation $T(\cdot;\param)$, such that if $\bsz\sim p_0$, then
 $T(\bsz;\param)\sim q_\param$. Then
\begin{align*}
\nabla \e_\param({f(\param)})=\e_{p_0}(\nabla  f(T(\bsz;\param))).
\end{align*}
It is often easy to sample from the base distribution $p_0$, and thus to
approximate the expectation by MC or RQMC samples.
This is the method we use in our examples.

\subsection{Bayesian linear regression}

We start with a toy example where we can find $\param^*$ analytically.
This will let us study $\Vert\param_k-\param^*\Vert$ empirically.
We consider the hierarchical linear model
\begin{align*}
\bsy\giv\beta\sim\dnorm(X\beta,\gamma^2I_N)
\quad\text{for}\quad
\beta\sim\dnorm(0,I_d)
\end{align*}
where $X\in\real^{N\times d}$ is a given matrix of full rank $d\le N$
and $\gamma^2\in(0,\infty)$ is a known error variance.
Here $N$ is the number of data points in our simulation
and not the number $n$ of MC or RQMC gradient evaluations.
The entries in $X$ are IID $\dnorm(0,1)$ random variables
and we used $\gamma =0.5$.


Translating this problem into the VB setup,
we make our latent variable $\bsz$ the unknown parameter vector $(\beta_1,\dots,\beta_d)$,
and we choose a very convenient variational distribution $q_\param$ with
$\beta_j\simind\dnorm(\mu_j,\sigma_j^2)$  for $j=1,\ldots,d$.
Now $\param = (\mu_1,\dots,\mu_d,\sigma_1,\dots,\sigma_d)$,
and $\bsy$ plays the role of the observations $\bsx$.
We also write $\mu = (\mu_1,\dots,\mu_d)$ and $\sigma=(\sigma_1,\dots,\sigma_d)$
for the parts of $\param$.

The ELBO has the expression
\begin{align*}
\calL(\param)
&=\e_\param({\log p(\bsy\giv\beta)})-\KL(q_\param(\beta)\,\|\, p(\beta\giv \bsy))\\
&=\e_\param({\log p(\bsy\giv\beta)})-\sum_{j=1}^d\Bigl(\frac{\sigma_j^2+\mu_j^2-1}{2}-\log\sigma_j\Bigr),
\end{align*}
where $\log p(\bsy\giv\beta)=-(p/2)\log(2\pi\gamma^2)-\|\bsy-X\beta\|_2^2/(2\gamma^2)$.
In this example, the ELBO has a
closed form and the optimal variational parameters are given by
\begin{align*}
\mu^*&=\Bigl(\frac{X\tran X}{\gamma^2}+I_d\Bigr)^{-1}\frac{X\tran \bsy}{\gamma^2}
\quad\text{and}\quad
\sigma^*_{j}=\Bigl(1+\frac{\|X_{\sumdot j}\|^2}{\gamma^2}\Bigr)^{-1/2},
\end{align*}
where $X_{\sumdot j}\in\real^N$ is the $j$'th column of $X$.

In this setting the Hessian is simply
$-X\tran X/\gamma^2$ and so stochastic quasi-Newton gradient
estimates are not needed.
We can however compare the effectiveness
of MC and RQMC in SGD.
We estimate the gradient by MC or RQMC
samples and use SGD via AdaGrad \citep{duchi2011adaptive}
to maximize the ELBO.

Our computations used  one example data set
with $N=300$ data points, $d=100$ variables and $K=1000$
iterations.
At each iteration, we draw a new sample of sample size $n$
of the $d$-dimensional Gaussian used to sample $\beta$.
The sample size $n$ is fixed
in each run, but we vary it between
over the range $8\le n\le 8192$ through powers
of $2$ in order to explore how quickly
MC and RQMC converge.

For RQMC, we use the
scrambled Sobol'  points implemented in PyTorch \citep{balandat2020botorch}
using the inverse Gaussian CDF $\psi(\cdot)=\Phi^{-1}(\cdot)$
to translate uniform random variables into standard Gaussians
that are then multiplied by $\sigma_j$ and shifted by $\mu_j$
to get the random $\beta_j$ that we need.
We compute the log errors
$\log_2\Vert \calL_k-\calL^*\Vert$ and $\log_2\|\param_k-\param^*\|$ and
average these over the last 50 iterations.
The learning rate in AdaGrad was taken to be $1$.

The results are shown in Figure \ref{fig: linear regression}.
We see there that RQMC achieves a
higher accuracy than plain MC. This happens because RQMC
estimates the gradient with lower variance.
In this simple setting the rate of convergence is
improved. \cite{balandat2020botorch} report similar
rate improvements in Bayesian optimization.

\begin{figure}
\centering
\begin{subfigure}{.32\textwidth}
\includegraphics[width=\textwidth]{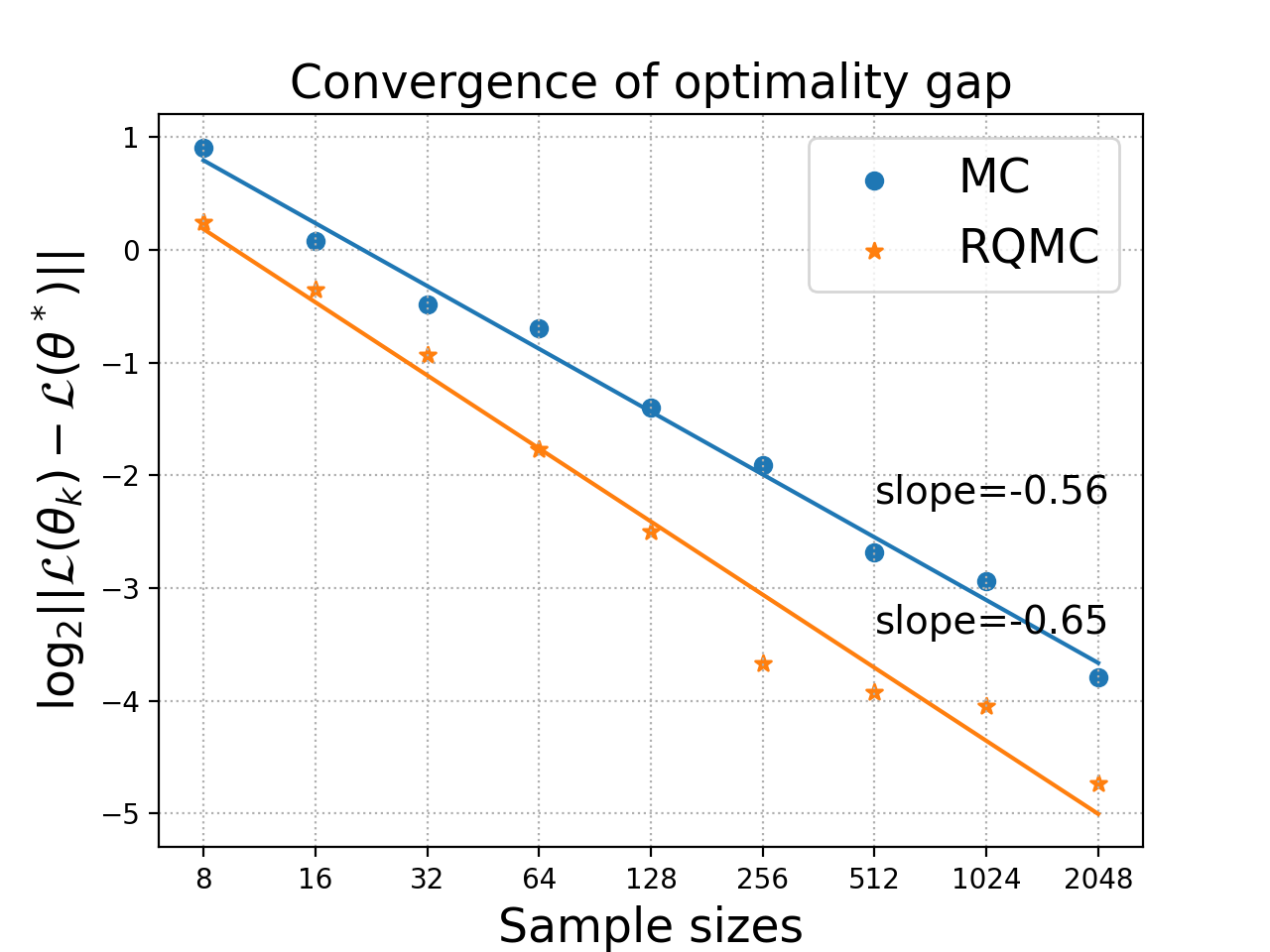}
\end{subfigure}
\begin{subfigure}{.32\textwidth}
\includegraphics[width=\textwidth]{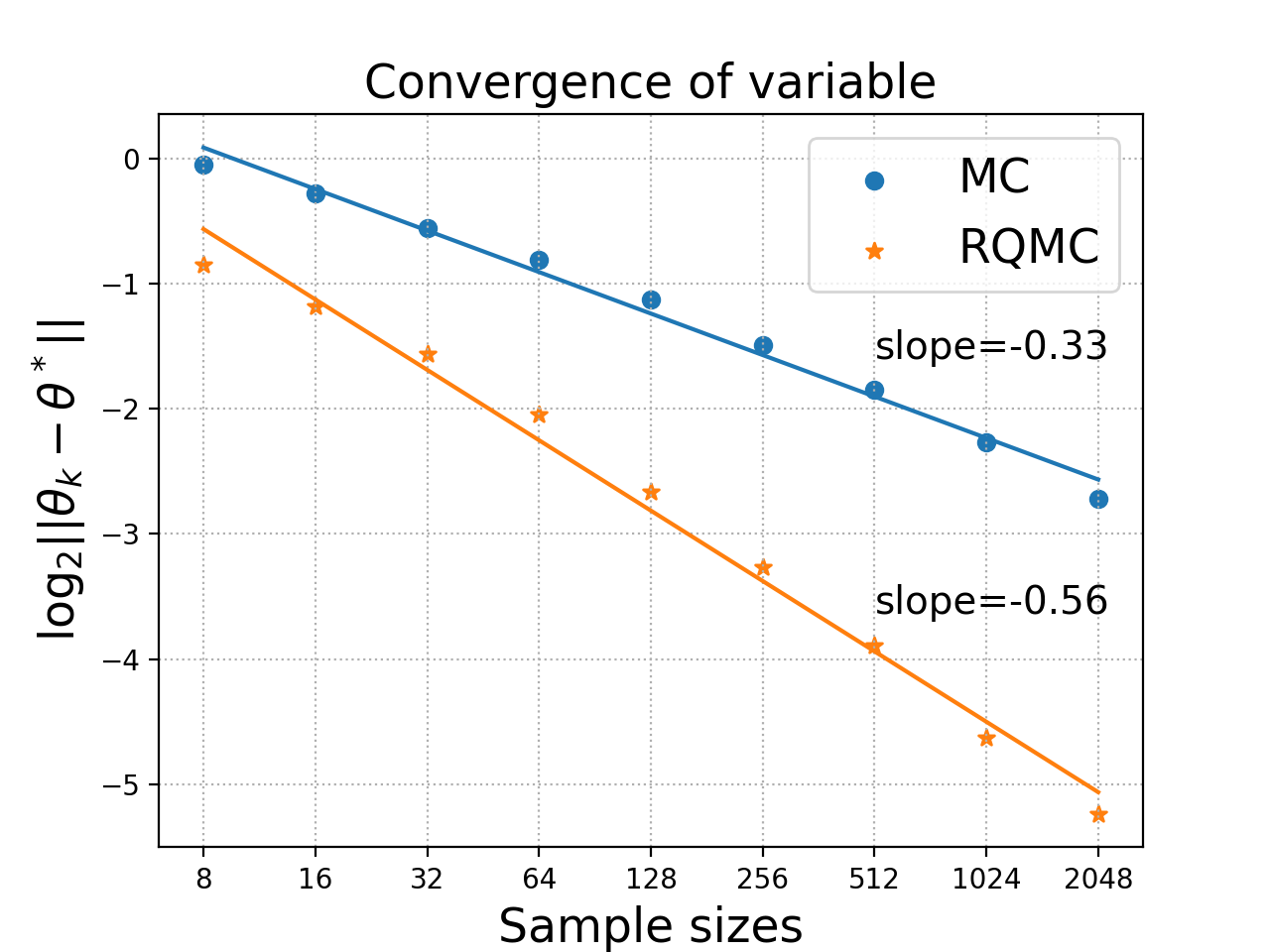}
\end{subfigure}
\begin{subfigure}{.32\textwidth}
\includegraphics[width=\textwidth]{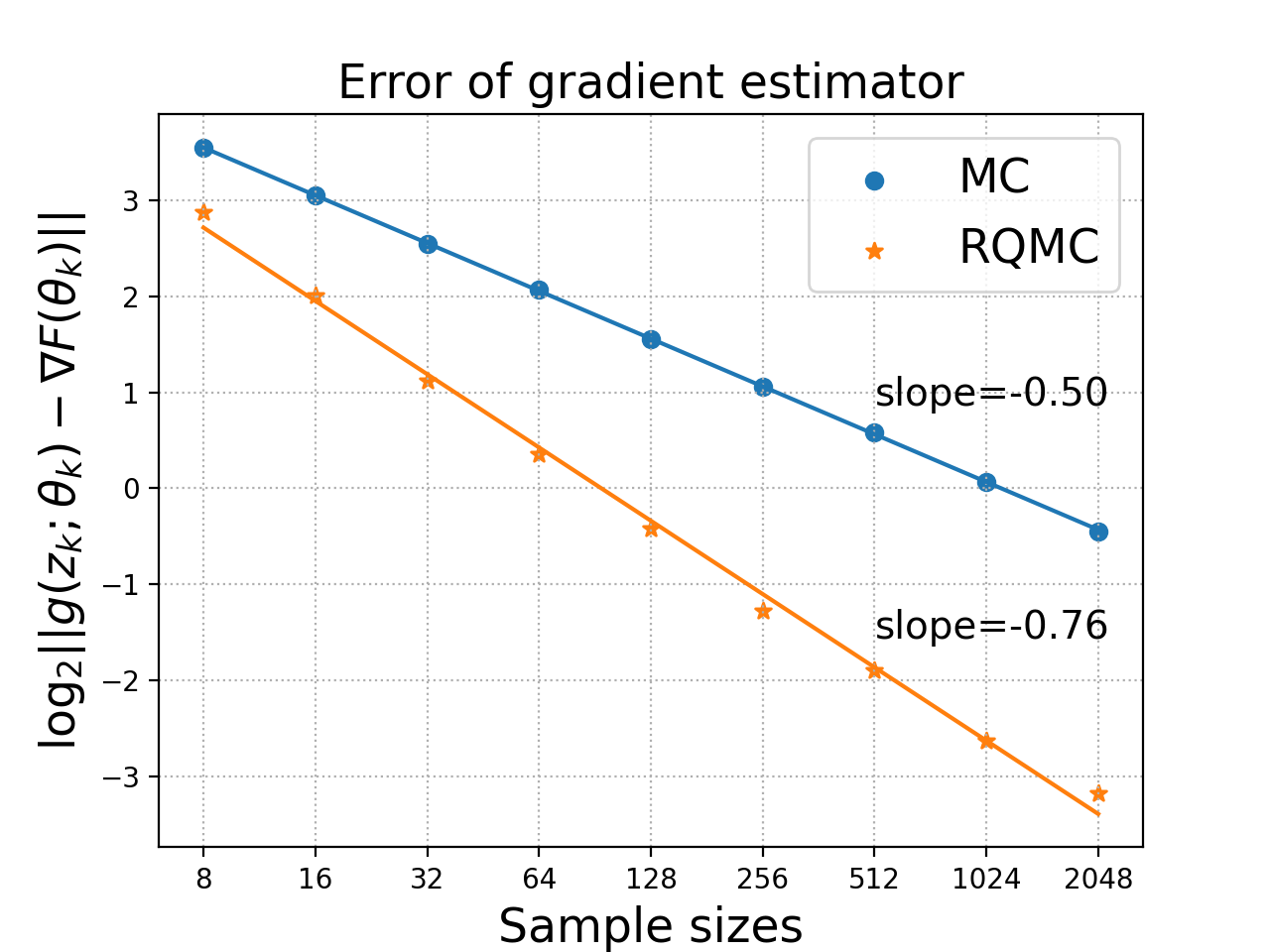}
\end{subfigure}
\caption{
The left panel has the average of
$\log_2|\calL(\param_k)-\calL(\param^*)|$
over the last $50$ values of $k$
versus $n$.  The middle panel has
that average of $\log_2\|\param_k-\param^*\|$.
The right panel has
the average of
$\log_2\Vert g(\bsz_k;\param_k)-\nabla F(\param_k)\Vert$
versus $n$ where $\hat g(\hat\param_k)=
(1/n)\sum_{i=1}^n g(\bsz_{k,i};\param_k)$.
The straight lines are least squares fits with
their slopes written above them.
}
\label{fig: linear regression}
\end{figure}

\subsection{Bayesian logistic regression}

Next we consider another simple example, though it is one
with no closed form expression for $\param^*$.
We use it to compare first and second order methods.
The Bayesian logistic regression model is defined by
\begin{align*}
\Pr(y_i=\pm1\giv \bsx_i,\beta)&=\frac{1}{1+\exp(\mp\bsx_i\tran\beta)},\quad i=1,\ldots,N
\quad\text{where}\quad
\beta\sim\dnorm(0,I_d).
\end{align*}
As before, $\beta$ is the unknown $\bsz$,
and $p_\param$ has $\beta_j\simind\dnorm(\mu_j,\sigma_j^2)$,
for $\param=(\mu_1,\dots,\mu_d,\sigma_1,\dots,\sigma_d)$.
The ELBO has the form
\begin{align*}
\calL_\param&=
\e_\param\biggl(\,{\sum_{i=1}^N \log S(y_i\bsx_i\tran\beta)}\biggr)
-\sum_{j=1}^d\Bigl(\frac{\sigma_j^2+\mu_j^2-1}{2}-\log\sigma_j\Bigr),
\end{align*}
where $S$ denotes the sigmoid function $S(x)=(1+e^{-x})^{-1}$.


In our experiments, we take $N=30$ and $d=100$.
With $d>N$ it is very likely that the data can be perfectly
linearly separated and then a Bayesian approach
provides a form of regularization.
The integral to be computed is
in $100$ dimensions, and the parameter to be optimized is in $200$ dimensions. We
generate the data from $\beta\sim\N(0,I_d/N)$,
then $\bsx_i\simiid\N(0,I_d)$ and finally
$y_i = 1$ with probability $1/(1+e^{-\bsx_i\tran\beta})$
are sampled independently for $i=1,\dots,N$.

In Figure \ref{fig: logreg}, we show the convergence of ELBO versus wall clock
times for different combinations of sampling methods (MC, RQMC) and
optimization methods (AdaGrad, L-BFGS). The left panel draws $n_g=8$
samples in each optimization iterations, while the right panel takes
$n_g=128$. The initial learning rate for AdaGrad is 0.01. The L-BFGS
is described in Algorithm \ref{alg: qmc-l-bfgs}, with
$n_h=1024$ Hessian evaluations
every $B=20$ steps with memory size $M=50$ and  $\alpha=0.01$.
Because L-BFGS uses some additional gradient function evaluations
to update the Hessian information
at every $B$'th iteration that the first order methods
do not use, we compare wall clock times.
The maximum iteration count in
the line search was 20. We used the Wolfe condition (Condition 3.6 in \cite{nocedal2006numerical}) with $c_1=0.001$ and $c_2=0.01$.

For this problem, L-BFGS always converges faster than AdaGrad. We can also see
that plain MC is noisier than RQMC.
The ELBOs for AdaGrad still seem to be increasing slowly even
at the end of the time interval shown.
For AdaGrad, RQMC consistently has a slightly higher ELBO than MC does.

\begin{figure}
\centering
\begin{subfigure}{.48\textwidth}
\includegraphics[width=\textwidth]{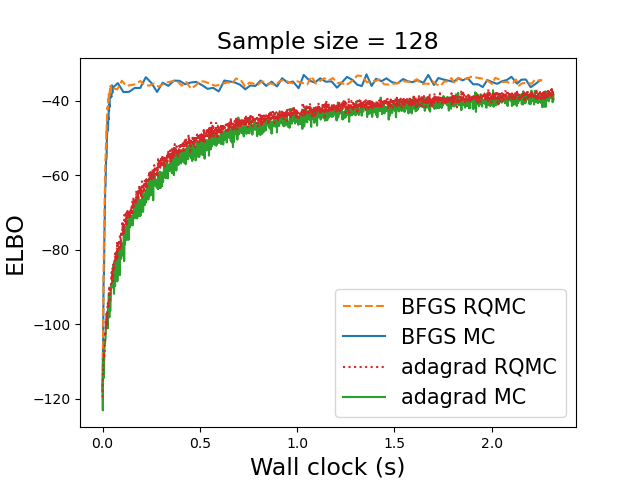}
\end{subfigure}
\begin{subfigure}{.48\textwidth}
\includegraphics[width=\textwidth]{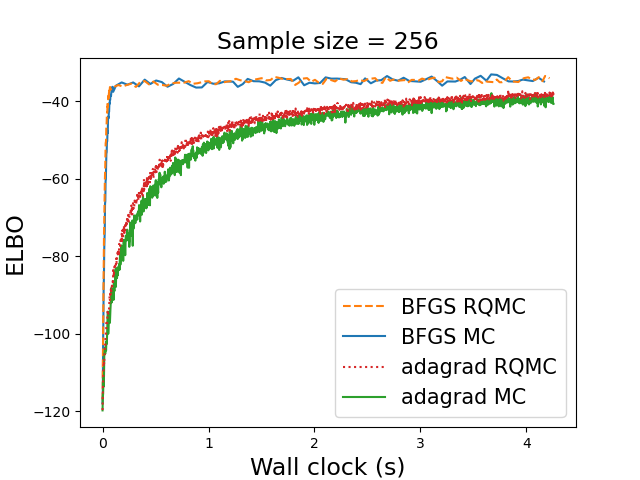}
\end{subfigure}
\caption{ELBO versus wall clock time in VB
for Bayesian logistic regression. The methods and setup
are described in the text. There are $n_g\in\{128,256\}$
gradient samples at each iteration and the second
order methods use $n_h = 1024$ Hessian samples
every $B=20$'th iteration.
}
\label{fig: logreg}
\end{figure}

\subsection{Crossed random effects}
In this section, we consider a crossed random effects model.
Both Bayesian and frequentist approaches to
crossed random effects can be a challenge
with costs scaling like $N^{3/2}$ or worse.
See \cite{papa:robe:zane:2020} and \cite{ghos:hast:owen:2020}
for Bayesian and frequentist approaches and also
 the dissertation of \cite{gao:thesis}.

An intercept only version of this model has
\begin{align*}
Y_{ij}&\simind\N(\mu + a_i + b_j,1),\quad 1\leq i\leq I,\quad 1\leq j\leq J
\end{align*}
given
$\mu\sim\N(0,1)$,
$a_i\simiid\N(0,\sigma_a^2)$,
and $b_j\simiid\N(0,\sigma_b^2)$
where $\log\sigma_a$ and
$\log\sigma_b$ are both
$\dnorm(0,1)$.  All of $\mu$, $a_i$, $b_j$ and the log
standard deviations are independent.

We use VB to approximate the posterior distribution of the
$d=I+J+3$ dimensional parameter
$\bsz=(\mu,\log\sigma_a,\log\sigma_b,\mathbf{a},\mathbf{b})$. In our example, we take $I=10$ and $J=5$.
As before $q(\bsz\giv\param)$ is chosen to be Gaussian with
independent coordinates and $\param$ has their means and standard deviations.
In Figure \ref{fig: cross effects}, we plot the
convergence of ELBO for different combinations of sampling methods and
optimization methods. The BFGS method takes $B=20$, $M=30$, $n_h=512$.
We used a learning rate of $0.01$ in AdaGrad. We observe that when the sample size is
8 (left), plain Monte Carlo has large fluctuations even when converged, especially
 for BFGS. When the sample size is 128 (right), the fluctuations disappear. But
RQMC still achieves a higher ELBO than plain Monte Carlo for BFGS. In both cases,
BFGS finds the optimum faster than AdaGrad.

\begin{figure}[t]
\centering
\begin{subfigure}{.48\textwidth}
\includegraphics[width=\textwidth]{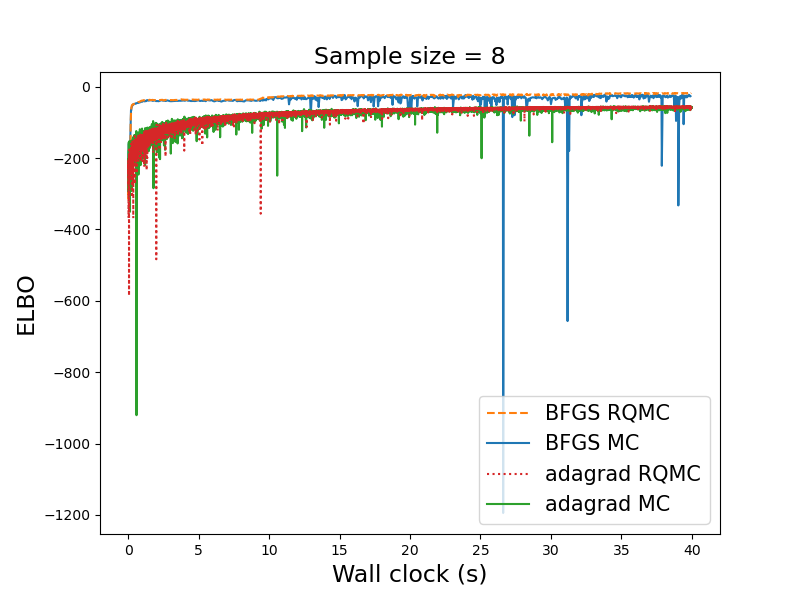}
\end{subfigure}
\begin{subfigure}{.48\textwidth}
\includegraphics[width=\textwidth]{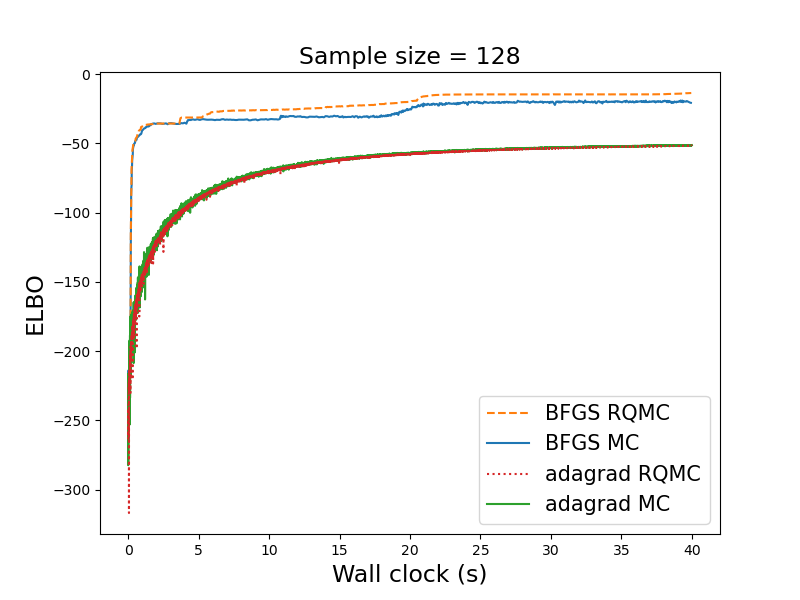}
\end{subfigure}
\caption{ELBO versus wall clock time in VB
for crossed random effects. The methods and setup
are described in the text. There are $n_g\in\{8,128\}$
gradient samples at each iteration.
}
\label{fig: cross effects}
\end{figure}

\subsection{Variational autoencoder}


A variational autoencoder (VAE, \cite{kingma2014auto}) learns a generative model
for a dataset. A VAE has a probabilistic \emph{encoder} and a probabilistic
\emph{decoder}. The encoder first produces a distribution $q_\phi(\bsz\giv\bsx)$ over the
latent variable $\bsz$ given a data point $\bsx$, then the decoder reconstructs a
distribution $p_\param(\bsx\giv\bsz)$ over the corresponding $\bsx$ from the latent variable $\bsz$.
The goal is to maximize the marginal probability $p_\param(\bsx)$. Observe that the
ELBO provides a lower bound of $\log(p_\param(\bsx))$:
\begin{align*}
\log p_{\param}(\bsx)-\KL(q_{\phi}(\bsz\giv\bsx)\,\|\, p_{\param}(\bsz\giv\bsx))
&=\e_\phi(\log p_\param(\bsx\giv\bsz)\giv\bsx)-\KL(q_{\phi}(\bsz\giv\bsx)\,\|\,p_\param(z))=:\calL(\param,\phi\giv\bsx),
\end{align*}
where $\e_\phi(\cdot\giv\bsx)$ denotes expection for random $\bsz$ given $\bsx$
with parameter $\phi$.
In this section $\bsz$ is the latent variable, and not a
part of the $\bigz$ that we use
in our MC or RQMC algorithms.
We do not refer to those variables in our VAE description below.

The usual objective is to maximize the ELBO
$\sum_{i=1}^N\calL(\param,\phi\giv\bsx_i)$ for a sample of $N$ IID $\bsx_i$
and now we have to optimize over $\phi$ as well as $\param$.
The first term $\e_\phi(\log p_\param(\bsx\giv\bsz))$
in the ELBO is the reconstruction
error, while the second term $\KL(q_{\phi}(\bsz\giv\bsx)\,\|\,p_\param(z))$
penalizes parameters $\phi$ that
give a posterior $q_\phi(\bsz\giv\bsx)$
too different from the prior $p_\param(\bsz)$.
Most commonly, $q_{\param}(\bsz\giv\bsx)=\N(\mu(\bsx;\param),\Sigma(\bsx;\param))$,
and $p_\param(\bsz)=\N(0,I)$,  so that the KL-divergence term
$\KL(q_{\phi}(\bsz\giv\bsx)\,\|\,p_\param(z))$
has a closed form. The decoding term $p_\param(\bsx\giv\bsz)$ is usually chosen to be a
 Gaussian or Bernoulli distribution, depending on the data type of $\bsx$.
The expectation $\e_\phi({\log p_\param(\bsx\giv\bsz)\giv\bsx})$ is
ordinarily estimated by  MC.
We implement  both plain MC and RQMC in our experiments.
To maximize the ELBO, the easiest way is to use
SGD or its variants. We also compare  SGD with L-BFGS in the experiments.

The experiment uses the MNIST dataset in PyTorch. It has 60,000
$28\times28$ gray scale images, and so
the dimension is 784. All experiments were conducted on a cluster node with
 2 CPUs and 4GB memory. The training was conducted in a mini-batch
manner with batch size 128. The encoder has a hidden layer with 800
nodes, and an output layer with 40 nodes, 20 for $\mu$ and 20 for $\sigma$.
Our $\Sigma(\bsx;\param)$ takes the form $\diag(\sigma(\bsx;\param))$.
The decoder has one hidden layer with 400 nodes.

In Figure \ref{fig: vae elbo}, we plot the ELBO versus wall clock time for
different combinations of sampling methods (MC, RQMC) and optimization
methods (Adam, BFGS).  The learning rate for Adam is 0.0001. For BFGS, the
memory size is $M=20$. The other tuning parameters are set to the defaults
from PyTorch. We observe that BFGS converged faster than Adam. For BFGS, we can
also see that RQMC achieves a slightly higher ELBO than MC. Figure
\ref{fig: recon bfgs rqmc} through \ref{fig: recon adam mc} shows some
reconstructed images using the four algorithms.

\begin{figure}[t]
  \centering
  \begin{tabular}[c]{cc}
    \multirow{4}{*}[14pt]{
    \begin{subfigure}{0.5\textwidth}
      \includegraphics[width=1.1\textwidth]{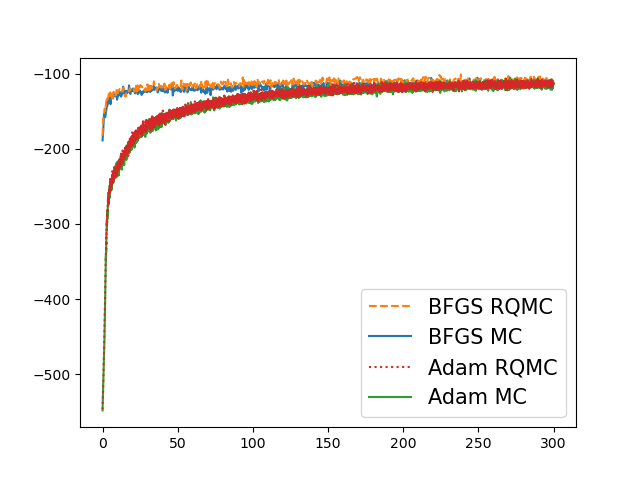}
      \caption{ELBO}
      \label{fig: vae elbo}
    \end{subfigure}
}&
   \begin{subfigure}[c]{0.39\textwidth}
      \includegraphics[width=\textwidth]{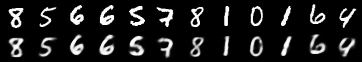}
      \caption{BFGS RQMC}
      \label{fig: recon bfgs rqmc}
    \end{subfigure}\\&
    \begin{subfigure}[c]{0.39\textwidth}
      \includegraphics[width=\textwidth]{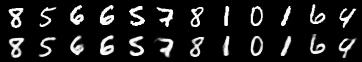}
      \caption{BFGS MC}
    \end{subfigure}\\&
    \begin{subfigure}[c]{0.39\textwidth}
      \includegraphics[width=\textwidth]{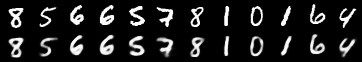}
      \caption{Adam RQMC}
    \end{subfigure}\\&
    \begin{subfigure}[c]{0.39\textwidth}
      \includegraphics[width=\textwidth]{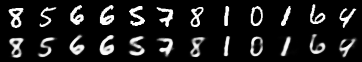}
      \caption{Adam MC}
      \label{fig: recon adam mc}
    \end{subfigure}
  \end{tabular}

\vspace*{.3cm} 
  \caption{
Plot (a) shows the ELBO versus wall clock time for
MC and RQMC used in both L-BFGS and Adam.
Plots (b) through (e) show example images.
}
  \label{fig: vae}
\end{figure}

\section{Discussion}\label{sec:discuss}

RQMC methods are finding uses in simulation optimization problems
in machine learning, especially in
first order SGD algorithms.
We have looked at their use in a second order, L-BFGS
algorithm.  RQMC is known theoretically and empirically
to improve the accuracy in integration problems
compared to both MC and QMC.
We have shown that improved estimation
of expected gradients translates directly into improved
optimization for quasi-Newton methods.
There is a small burden in reprogramming algorithms
to use RQMC instead of MC, but that is greatly mitigated
by the appearance of RQMC algorithms in
tools such as BoTorch \citep{balandat2020botorch}
and the forthcoming scipy 1.7 ({\tt scipy.stats.qmc.Sobol})
and QMCPy at \url{https://pypi.org/project/qmcpy/}.

Our empirical examples have used VB.
The approach has potential value in Bayesian
optimization \citep{frazier2018tutorial}
and optimal transport \citep{el2012bayesian,bigoni2016adaptive}
as well.

The examples we chose were of modest scale
where both first and second order methods could be used.
In these settings, we saw that second order methods improve
upon first order ones. For the autoencoder problem
the second order methods converged faster
than the first order ones did.
This also happenend for the crossed random effects problem
where the second order methods found better ELBOs than
the first order ones did and
RQMC-based quasi-Newton algorithm found a better ELBO than
the MC-based quasi-Newton did without increasing the wall clock time.

It is possible that RQMC will bring an advantage to conjugate
gradient approaches as they have some similarities to L-BFGS.
We have not investigated them.

\section*{Acknowledgments}

This work was supported by the National Science Foundation
under grant IIS-1837931.

\bibliographystyle{apalike}
\bibliography{ref,qmc}

\vfill\eject
\appendix
\section{Proof of main theorems}
Our approach is similar to that used by \cite{buchholz2018quasi}.
They studied RQMC with SGD whereas we consider
L-BFGS, a second order method.


\subsection{Proof of Theorem \ref{thm: f}}
\label{sec: proof 1}
\begin{proof}
Let $e_k=\bar g(\bigz_k;\param_k)-\nabla F(\param_k)$ be the error in estimating the gradient at step $k$.
By the unbiasedness assumption, $\e(e_k\giv\param_k)=0$.
Starting from the Lipschitz condition, we have
\begin{align*}
F(\param_{k+1})-F(\param_k)&\leq\nabla F(\param_k)\tran(\param_{k+1}-\param_k)+\frac{L}{2}\|\param_{k+1}-\param_k\|^2\\
&=-\alpha_k\nabla F(\param_k)\tran H_k \bar g(\bigz_k;\param_k)+\frac{L\alpha_k^2}{2}\|H_k \bar g(\bigz_k;\param_k)\|^2\\
&=-\alpha_k\nabla F(\param_k)\tran H_k e_k-\alpha_k \nabla F(\param_k)\tran {H_k} \nabla F(\param_k)\\
&\qquad+\frac{L\alpha_k^2}{2}\left(\|H_k e_k\|^2+\|H_k \nabla F(\param_k)\|^2 + 2\nabla F(\param_k)\tran H_k^2 e_k\right)\\
&\leq -\alpha_k\nabla F(\param_k)\tran H_k e_k - \alpha_k h_1\|\nabla F(\param_k)\|^2  \\
&\qquad +\frac{L\alpha_k^2}{2} (h_2^2\|e_k\|^2+ h_2^2 \|\nabla F(\param_k)\|^2+ 2\nabla F(\param_k)\tran H_k^2 e_k)\\
&=-\alpha_k \nabla F(\param_k)\tran H_k e_k + L\alpha_k^2 \nabla F(\param_k)\tran H_k^2 e_k
\\
&\qquad -\alpha_kh_1\Bigl(1-\frac{L\alpha_k h_2^2}{2h_1}\Bigr)\|\nabla F(\param_k)\|^2+\frac{L\alpha_k^2h_2^2}{2}\|e_k\|^2.
\end{align*}
Because $\alpha_k=\alpha \leq {h_1}/({Lh_2^2})$, we have $1-{L\alpha_kh_2^2}/({2h_1})\geq1/2$.
Because strong convexity implies
\[
\|\nabla F(\param)\|^2\geq 2c(F(\param)-F^*),\quad \forall \param,
\]
we have
\begin{align*}
F(\param_{k+1})-F(\param_k)&\leq -\alpha_k \nabla F(\param_k)\tran (H_k-L\alpha_k H_k^2) e_k  -\alpha_kh_1c(F(\param_k)-F^*)+\frac{L\alpha_k^2h_2^2}{2}\|e_k\|^2.
\end{align*}
Adding $F(\param_k)-F^*$ to both sides gives
\begin{align*}
F(\param_{k+1})-F^*&\leq (1-\alpha_kh_1c)(F(\param_k)-F^*)+R_k,\numberthis\label{equ: recursive relation}
\end{align*}
where
\begin{align*}
R_k=-\alpha_k \nabla F(\param_k)\tran (H_k-L\alpha_k H_k^2) e_k +\frac{L\alpha_k^2h_2^2}{2}\|e_k\|^2.
\end{align*}
Let $\calF_k=\sigma(\bigz_i,1\leq i\leq k)$ be the filtration generated by the random
inputs $\{\bigz_k\}$ to our sampling process.
Because $\bigz_k$ are mutually independent and $H_k$ is
independent of $\bigz_k$,
we have $\param_k,H_k\in\calF_{k-1}$ and $\e(e_k\giv\calF_{k-1})=0$.
Then
\begin{align*}
\e\bigl(\nabla F(\param_k)\tran (H_k-L\alpha_k H_k^2) e_k\giv \calF_{k-1}\bigr)&=\nabla F(\param_k)\tran (H_k-L\alpha_k H_k^2)\e(e_k\giv \calF_{k-1})=0.
\end{align*}
Therefore, $\nabla F(\param_k)\tran (H_k-L\alpha_k H_k^2) e_k $ is a
martingale difference sequence w.r.t.\ $\calF_k$.
Let
$$V_k=\e(\|e_k\|^2\giv\param_k)=\tr\bigl(\Var{\bar g(\bigz_k;\param_k)\giv\param_k}\bigr).$$
Then $\|e_k\|^2-V_k$ is also a martingale difference sequence w.r.t. $\calF_{k}$.
So we can write
\begin{align*}
R_k=\nu_k+\frac{L\alpha_k^2h_2^2}{2}V_k,
\end{align*}
where
\begin{align*}
\nu_k=-\alpha_k \nabla F(\param_k)\tran (H_k-L\alpha_k H_k^2) e_k +\frac{L\alpha_k^2h_2^2}{2}(\|e_k\|^2-V_k)
\end{align*}
is a martingale difference sequence, and
${L\alpha_k^2h_2^2}/({2}V_k)$ is a deterministic sequence.
Recursively applying equation~\eqref{equ: recursive relation} gives
\begin{align*}
F(\param_{K})-F^*&\leq(1-\alpha ch_1)^{K}(F(\param_0)-F^*)+\sum_{k=0}^{K-1}(1-\alpha c h_1)^{K-k-1}R_k\\
&=(1-\alpha c h_1)^{K}(F(\param_0)-F^*)+\sum_{k=0}^{K-1}(1-\alpha c h_1)^{K-k-1}
\Bigl(\nu_k+\frac{L\alpha^2 h_2^2}{2}V_k\Bigr).
\end{align*}
By the bounded variance assumption, $V_k\leq M$ for all $k\geq0$. Hence,
\begin{align*}
F(\param_{K})-F^*&\leq(1-\alpha c)^{K}(F(\param_0)-F^*)+\frac{\alpha Lh_2^2}{2ch_1}M+\sum_{k=0}^{K-1}(1-\alpha ch_1)^{K-k-1}\nu_k.
\end{align*}
Taking expectations on both sides proves \eqref{eq: bound expectation}.

To prove the finite sample guarantee \eqref{eq: bound finite sample}, it remains to bound the martingale $\sum_{k=0}^{K-1}(1-\alpha ch_1)^{K-k-1}\nu_k$ with high probability.
We assumed a bound on $\Vert g(\bsz;\param)\Vert$
which implies one for $\Vert \bar g(\bigz;\param)\Vert$ as well for any fixed $n$.
When the norms of the gradient $\nabla F(\param)$ and gradient
estimator $\bar g(\bigz;\param)$ are bounded by such a constant $C$ for all $\param$ and $\bigz$, then $\|e_k\|=\|\bar g(\bigz;\param)-\nabla F(\param) \|\leq 2C$, and
we have the bound
\begin{align*}
|\nu_k|\leq \alpha |\nabla F(\param_k)\tran (H_k-L\alpha H_k^2)e_k| + \frac{L\alpha^2h_2^2}{2}C^2 \leq 2\alpha (h_2-L\alpha h_1^2+{L\alpha h_2^2})C^2=:C',
\end{align*}
where the second inequality uses that the largest eigenvalue of $H_k-L\alpha H_k^2$ is upper bounded by $h_2-L\alpha h_1^2$.
By the Azuma-Hoeffding inequality \citep{azuma1967weighted}, for all $t\geq 0$,
\begin{align*}
\pr\biggl(\,{\sum_{k=1}^K(1-\alpha ch_1)^{K-k-1}\nu_k\geq t}\biggr)
\leq \exp\biggl({-\frac{2t^2}{\sum_{k=1}^K(1-\alpha c h_1)^{K-k-1}C'^2}}\biggr)
\leq \exp\biggr({-\frac{2t^2}{\frac{C'^2}{\alpha c h_1}}}\biggr).
\end{align*}
Setting $\ep^2={2t^2\alpha c h_1}/{C'^2}$ gives
\begin{align*}
t=\frac{C'}{\sqrt{2\alpha ch_1}}\ep=\frac{2\alpha C^2 (h_2-L\alpha h_1^2+{L\alpha h_2^2})}{\sqrt{2\alpha ch_1}}\ep\leq C^2\sqrt{\frac{2\alpha}{ch_1}}(h_2-L\alpha h_1^2+{h_1}{})\ep.
\end{align*}
So we have proved that with probability at least $1-e^{-\ep^2}$,
\begin{align*}
F(\param_{K})-F^*&\leq(1-\alpha c h_1)^{K}(F(\param_0)-F^*)+\frac{\alpha L h_2^2}{2ch_1}M+C^2\sqrt{\frac{2\alpha}{ch_1}}(h_2-L\alpha h_1^2+{h_1}{})\ep.
\end{align*}
\end{proof}

\subsection{Proof of Theorem \ref{thm: param}}
\label{sec: proof 2}

Our proof uses the following lemma.
\begin{lemma}\label{lem:udvbound}
Let $u,v\in\real^n$ satisfy
 $u\tran v\geq A$ and let $D\in\real^{n\times n}$ be symmetric
with $h_1I\preccurlyeq D\preccurlyeq h_2I$ where $0<h_1\leq h_2$.
Then
$$u\tran Dv\geq \frac{h_1+h_2}{2}A-\frac{h_2-h_1}{2} \|u\|\|v\|.$$
\end{lemma}
\begin{proof}
Without loss of generality, we can assume that $D$ is a diagonal matrix.
Otherwise, let $D=U\Lambda U\tran$ be the eigen-decomposition of $D$.
Then $u\tran Dv=(U\tran u)\tran\Lambda  (U\tran v)$,
 $\|U\tran u\|=\|u\|$, $\|U\tran v\|=\|v\|$, and $(U\tran u)\tran (U\tran v)=u\tran v$,
and we can replace $D$ by $\Lambda$.

Let $d_1,\ldots,d_n$ be the diagonal entries of $D$. Then $u\tran Dv=\sum_{i=1}^nu_iv_id_i$. Let $s_+=\sum_{i:u_iv_i\geq 0}u_iv_i$ and $s_-=\sum_{i:u_iv_i<0}u_iv_i$. Note that $s_+ + s_-=u\tran v\geq A$, $s_+-s_-=\sum_{i=1}^n|u_iv_i|\leq \|u\|\|v\|$.
Then for any $D$,
\[
u\tran Dv\geq h_1 s_+ + h_2s_-=\frac{h_1+h_2}{2}(s_+ + s_-)+\frac{h_2-h_1}{2}(s_- - s_+)\geq \frac{h_1+h_2}{2}A-\frac{h_2-h_1}{2}\|u\|\|v\|.
\]

\end{proof}

\par\noindent
Now we are ready to prove Theorem~\ref{thm: param}.
\begin{proof}
We start by decomposing
\begin{align*}
\|\param_{k+1}-\param^*\|^2&=\|\param_{k+1}-\param_k+\param_k-\param^*\|^2\\
&=\|\param_k-\param^*\|^2+\alpha^2\|H_k \bar g(\bigz_k;\param_k)\|^2-2\alpha (\param_k-\param^*)\tran H_k \bar g(\bigz_k;\param_k).
\end{align*}
Note that only $\bar g(\bigz_k;\param_k)$ and $\param_{k+1}$ depend on $\bigz_k$.
Taking expectation w.r.t.\ $\bigz_k$ on both sides gives
\begin{align*}
\e\bigl(\|\param_{k+1}-\param^*\|^2\bigr)&\leq \|\param_k-\param^*\|^2+\alpha^2h_2^2 (M+\|\nabla F(\param_k)\|^2)-2\alpha(\param_k-\param^*)\tran H_k \nabla F(\param_k)\\
&\leq \|\param_k-\param^*\|^2+\alpha^2h_2^2(M+L^2\|\param_k-\param^*\|^2)-2\alpha (\param_k-\param^*)\tran H_k \nabla F(\param_k).\numberthis\label{eq: bound param 1}
\end{align*}
By strong convexity of $F(\cdot)$,
\[
(\param_k-\param^*)\tran \nabla F(\param_k)\geq F(\param_k) -F^* +\frac{c}{2}\|\param_k-\param^*\|^2\geq c\|\param_k-\param^*\|^2.
\numberthis\label{equ: inner prod}
\]

Using Lemma~\ref{lem:udvbound} and equation~\eqref{equ: inner prod}, we have
\begin{align*}
(\param_k-\param^*)\tran H_k \nabla F(\param_k) &\geq \frac{h_1+h_2}{2}c\|\param_k-\param^*\|^2-\frac{h_2-h_1}{2}\|\param_k-\param^*\| \|\nabla F(\param_k)\| \\
&\geq \Bigl(\frac{h_1+h_2}{2}c-\frac{h_2-h_1}{2}L\Bigr)\|\param_k-\param^*\|^2,
\end{align*}
where the last inequality is due to $\|\nabla F(\param_k)\|\leq L\|\param_k-\param^*\|$.
Combining this with \eqref{eq: bound param 1} gives
\begin{align*}
\e\bigl({\|\param_{k+1}-\param^*\|^2}\bigr)
&\leq \bigl(1+\alpha^2L^2h_2^2-\alpha [(h_1+h_2)c- (h_2-h_1)L]\bigr)\|\param_k-\param^*\|^2+\alpha^2h_2^2M.
\end{align*}
We have assumed that
$$0<\alpha <\frac{(h_1+h_2)c-(h_2-h_1)L}{2L^2h_2^2}$$
and it then follows that
$$\bigl(1+\alpha^2L^2h_2^2-\alpha [(h_1+h_2)c- (h_2-h_1)L]\bigr)
\leq 1-\alpha^2h_2^2L^2
$$
from which
\begin{align}\label{eq:errorpropagation}
\e\bigl({\|\param_{k+1}-\param^*\|^2}\bigr)
&\leq
\bigl(1-\alpha^2h_2^2L^2\bigr)
\|\param_k-\param^*\|^2+\alpha^2h_2^2M.
\end{align}
Applying the recursive error formula~\eqref{eq:errorpropagation}
we get
\begin{align*}
\e\bigl(\|\param_{k}-\param^*\|^2\bigr)
\leq (1-\alpha^2h_2^2L^2)^k\|\param_0-\param^*\|^2+\frac{M}{L^2}.
\end{align*}
where the expectation is over $\bigz_1,\dots,\bigz_k$.
\end{proof}

\end{document}